\documentclass{article}

\usepackage[preprint]{neurips_2026}

\usepackage[utf8]{inputenc} 
\usepackage[T1]{fontenc}    
\usepackage{hyperref}       

\usepackage{url}            
\usepackage{booktabs}       
\usepackage{amsfonts}       
\usepackage{nicefrac}       
\usepackage{microtype}      
\usepackage{xcolor}         
\definecolor{revisiongreen}{RGB}{0,0,0}

\usepackage{microtype}
\usepackage{graphicx}
\usepackage{subcaption}

\usepackage{amsmath}
\usepackage{amssymb}
\usepackage{mathtools}
\usepackage{amsthm}
\usepackage{comment}
\usepackage{algorithm}
\usepackage{algpseudocode}

\usepackage{float}

\theoremstyle{plain}
\newtheorem{theorem}{Theorem}[section]
\newtheorem{proposition}[theorem]{Proposition}

\theoremstyle{definition}

\theoremstyle{remark}

\usepackage{enumitem}
\usepackage{tcolorbox}

\newcommand{\R}{\mathbb R}

\newcommand{\ip}[2]{\langle #1,#2\rangle_F}
\newcommand{\norm}[1]{\lVert #1\rVert_F}

\newcommand{\tr}{\operatorname{tr}}

\newcommand{\ourname}{MuonGS}
 
\title{MuonGS: Pretraining LLMs with Relaxed Weight-Norm Control}

\author{
  Athanasios Glentis\thanks{Equal contribution, alphabetical order.}
  \qquad
  Dawei Li\footnotemark[1]
  \qquad
  Mingyi Hong\\ \\
  University of Minnesota\\
  \texttt{\{glent007,li004678,mhong\}@umn.edu}
}

\begin{document}

\maketitle

\begin{abstract}
\begingroup
Weight decay and explicit weight-norm constraints improve optimizer performance in language model pretraining, yet a unified understanding of why such control helps and how much is needed remains incomplete. We analyze norm control in Muon over the entire training trajectory using the \emph{accumulated effective learning rate}, $C_T=\sum_{t<T}\eta_t/\norm{W_t}$, where $\eta_t$ is the nominal learning rate and $\eta_t/\norm{W_t}$ is the effective learning rate relative to the weight norm. For bounded update directions, accumulation over any training interval bounds relative weight displacement. We show that vanilla Muon can accumulate only logarithmically, potentially limiting exploration and helping explain its slower empirical progress, whereas MuonW and MuonH guarantee at least linear accumulation under the stated conditions. Going beyond bounded-norm control, we show that sufficiently slow norm growth guarantees faster accumulation than the logarithmic worst case, opening up less restrictive optimizer designs, with more relaxed weight-norm control.

Building on this connection between norm growth and exploration, we propose MuonGS, which uses Gram--Schmidt (GS) projection to remove the radial component of each Muon update without fixing the weight norm. With bounded update directions and a constant nominal learning rate, we prove that the weight norm grows at most proportionally to $\sqrt{t}$ after $t$ steps. Consequently, MuonGS guarantees $C_T=\Omega(\sqrt{T})$, with accumulation sustained even in later training. Further, for smooth scale-invariant objectives we establish its $O(T^{-1/2})$ stationarity guarantee. In pretraining experiments with 160M- and 400M-parameter dense models as well as a 135M active- and 1B total- parameter MoE model, MuonGS achieves lower validation perplexity than MuonH at long training horizons. In particular, for the dense model it achieves an estimated token-equivalent speedup up to about $15\%$ at matched validation loss to MuonH.

\endgroup
\end{abstract}

\section{Introduction}

Optimizer design plays an important role in large language model (LLM) pretraining, shaping learning dynamics and model quality~\mbox{\citep{liu2025muon,wen2026fantastic}}. Understanding the principles behind successful optimizers is therefore essential for improving training.

Weight-norm control is a common component of effective optimizers for LLM pretraining. AdamW and the weight-decay variant of Muon, MuonW, counter norm growth through decoupled shrinkage~\mbox{\citep{loshchilov2017decoupled,liu2025muon}}. Its strength must be tuned, and the resulting weight scale depends on the optimization dynamics. Explicit constraints provide more direct control, as in fixed-radius optimization~\mbox{\citep{kodryan2022training}}, unit-norm weight vectors in nGPT~\mbox{\citep{loshchilov2024ngpt}}, and fixed weight and update norms in MuonH, the Hyperball variant of Muon~\mbox{\citep{wen2026fantastic}}. These constraints prescribe weight scales throughout training and leave less freedom for their evolution.

These different norm controls raise a common question: \emph{When is norm control needed, why does it help, and how much is sufficient?} Existing analyses explain important effects of weight norms on optimization~\mbox{\citep{zhang2018three,dangelo2024why,liu2026effective}}, but a unified basis for choosing among these controls remains incomplete. In particular, are the restrictions imposed by existing norm-control methods essential, or can weaker control deliver better optimization performance?

A common rationale for controlling weight norms is to prevent updates from becoming too small relative to the current weights. Following prior work~\mbox{\citep{dangelo2024why,liu2026effective}}, the \emph{effective learning rate} is the ratio $\eta_t/\norm{W_t}$ of the nominal learning rate $\eta_t$ to the Frobenius norm of the current weight matrix $W_t$. It determines the step size relative to the weight norm, up to the norm of the update direction. For a Muon update $W_{t+1}=W_t-\eta_t O_t$,
\[
\frac{\norm{W_{t+1}-W_t}}{\norm{W_t}}
=\frac{\eta_t}{\norm{W_t}}\norm{O_t}.
\]
Keeping the weight norm below $R$ ensures $\eta_t/\norm{W_t}\ge\eta_t/R$. When update norms are fixed, this protects relative movement in each step from shrinking due to norm growth. But must norm control preserve relative movement at every step, or is it sufficient to preserve exploration capacity over the training trajectory as a whole?

To examine this question, we begin with the measure of {\it net displacement} $\norm{W_T-W_0}$. Clearly, the same displacement produces a smaller relative change in the parameters when their norm is larger. We therefore normalize by the average weight norm $\bar r_T=T^{-1}\sum_{t<T}\norm{W_t}$, which summarizes the typical scale along the trajectory. This ratio measures net movement relative to that scale and is unchanged by a common rescaling of all iterates. For nonzero iterates, a constant nominal learning rate $\eta$, and $\norm{O_t}\le M$,
\begin{align}
\frac{\norm{W_T-W_0}}{\bar r_T}
\le\frac{M\eta T}{\bar r_T}
\le M\sum_{t<T}\frac{\eta}{\norm{W_t}}.
\end{align}
The first bound sums update lengths; the second uses the fact that the average inverse norm is at least the inverse average norm. Clearly, {\it sufficient} accumulation of effective learning rates is therefore a \emph{necessary condition} for substantial movement.

{\bf Our contribution.} We develop a {\it global} perspective on norm control for Muon-type algorithms by examining exploration over entire training trajectories. This perspective identifies less restrictive norm-control conditions and guides the design of \emph{MuonGS}, a new optimizer with theoretical guarantees and strong empirical performance. Central to our analysis is the \emph{accumulated effective learning rate}, defined for general learning-rate schedules by
\[
C_T=\sum_{t=0}^{T-1}\frac{\eta_t}{\norm{W_t}}.
\]
Our main claim is that continued exploration requires sufficient growth of $C_T$ over the training intervals of interest. Further, we extend the displacement argument to arbitrary intervals and learning-rate schedules, that for bounded update directions, net weight displacement from iteration $s$ to $T$ , relative to the larger endpoint norm, is at most $M(C_T-C_s)$. A large $C_T$ may come mostly from early iterations. Continued exploration after iteration $s$ requires sufficient additional accumulation, $C_T-C_s$, and if this increment is small, relative weight displacement over $[s,T]$ must also be small.

We next compare how different norm-control mechanisms affect this accumulation. When the learning rate $\eta_t=\eta$ is held constant throughout training, vanilla Muon can exhibit linear norm growth and only logarithmic accumulation, $C_T=\Theta(\log T)$. MuonW guarantees $C_T=\Omega(T)$ under the stated conditions, while fixed-norm control gives $C_T=\Theta(T)$. These results show that weight decay and fixed-norm constraints guarantee at least linear accumulation, compared with vanilla Muon's logarithmic worst case.
A direct consequence of our analysis is that bounded weight norms are {\it sufficient} for these linear accumulation guarantees, but stronger than needed to improve on vanilla Muon's logarithmic worst case. We show that this requirement can be relaxed to sufficiently slow norm growth, which still guarantees faster accumulation (as compared to vanilla Muon) even as weight norms increase and effective learning rates decrease. This observation opens up a broader family of Muon variants that control norm growth without requiring bounded weight norms. Can this additional flexibility be used to design optimizers that sustain accumulation throughout training and achieve better practical performance?

To investigate this question, we propose \emph{MuonGS}, which uses Gram--Schmidt (GS) projection to remove the radial component of Muon's update. The resulting update is orthogonal to the current weights in the Frobenius inner product. This orthogonality eliminates the first-order increase in the weight norm, so the norm grows only at second order in the step size. With bounded directions and a constant nominal learning rate, this guarantees $O(\sqrt{t})$ norm growth and $C_T=\Omega(\sqrt{T})$, with the same accumulation order over later training windows. MuonGS thus improves vanilla Muon's worst-case accumulation guarantee without fixing a radius. For smooth scale-invariant objectives, we further established the convergence rate for MuonGS.

\begin{figure*}[t]
  \centering
  \includegraphics[width=1\textwidth]{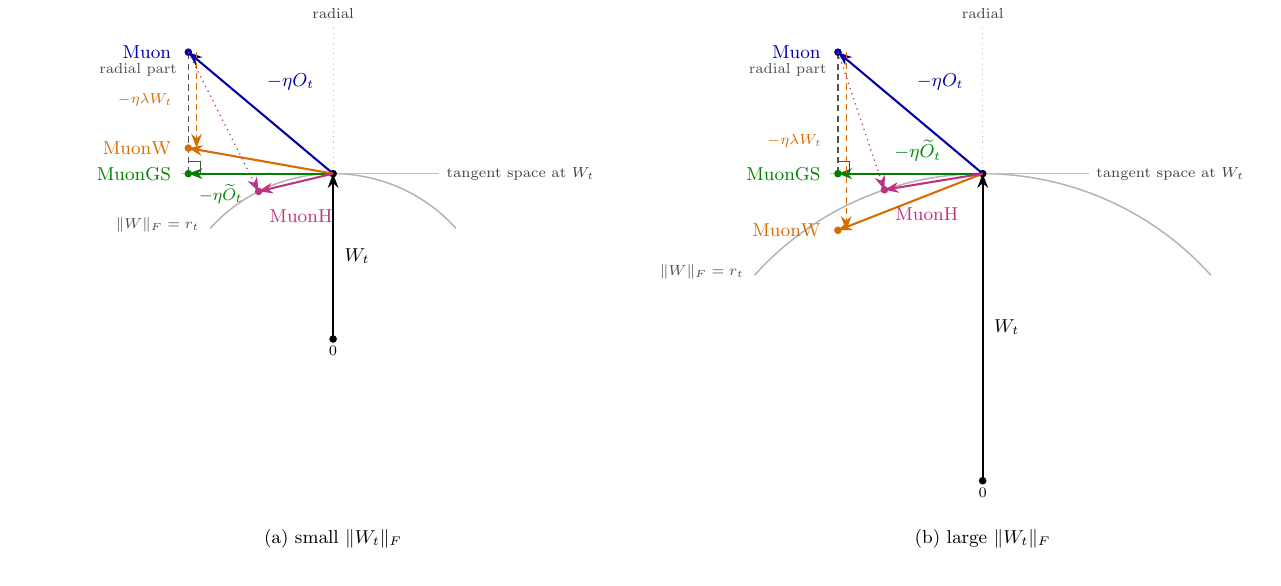}
  \caption{One step of vanilla Muon, MuonW, MuonH and MuonGS from the same $W_t$ and the same direction $O_t$, for a small (left) and a large (right) weight norm. Vanilla Muon’s radial part (dashed) grows the norm at first order; MuonH retracts to $\|W\|_F=R$ (dotted); MuonGS keeps the tangential part only, so the norm changes at second order. MuonW’s decay $-\eta\lambda W_t$ scales with $\|W_t\|_F$ but the actual step depends on the size of $W_t$; the decay therefore hardly slows down the norm growth for a small norm but can shrink the norm for a large one.}
  \label{fig:alg_compare}
\end{figure*}

We evaluate MuonGS on FineWeb pretraining with modernized GPT-style models of approximately 160M and 400M parameters. At longer training horizons, MuonGS achieves lower validation perplexity than MuonH. Fitting MuonH's loss as a function of the training-token budget gives an estimated token-equivalent speedup of up to about $15\%$ at matched validation loss. In addition, MuonGS obtains lower validation perplexity than MuonH when pretraining a modern MoE model, having about 135M active and 1B total parameters, on Ultra-FineWeb. These results support MuonGS as a practical optimizer and demonstrate the value of relaxing norm-control requirements. Together with our theoretical guarantees, these findings provide a positive answer to the question above by showing that weaker norm control can indeed sustain accumulation throughout training while improving empirical performance.

Our contributions are threefold.
\begin{itemize}[leftmargin=*]
\item We establish a global perspective on norm control in Muon, in the sense that we assess exploration over entire training trajectories through the growth of $C_T$ and its increments over training intervals. Our analysis bounds relative weight displacement and, under scale-invariance and Lipschitz assumptions, attainable loss change. At a constant nominal learning rate, it distinguishes vanilla Muon's logarithmic worst case from the linear guarantees of MuonW and fixed-norm control. It also explains when norm growth can offset a larger nominal learning rate.
\item We propose MuonGS and prove $O(\sqrt{t})$ weight-norm growth, $C_T=\Omega(\sqrt{T})$, and $C_{2T}-C_T=\Omega(\sqrt{T})$ for bounded directions and a constant nominal learning rate. These guarantees improvement on vanilla Muon's logarithmic worst-case accumulation without requiring bounded weight norms. The bound on $C_{2T}-C_T$ further shows that accumulation is sustained over later training windows, preserving a necessary condition for continued exploration. 
\item We demonstrate MuonGS's effectiveness in dense transformer pretraining at 160M and 400M parameter scale, with lower validation perplexity than MuonH at longer horizons and an estimated token-equivalent speedup of up to about $15\%$ at matched validation loss. In addition, we show MuonGS obtains lower validation perplexity than MuonH when pretraining a modern MoE model having about 135M active and 1B total parameters.
\end{itemize}

\section{Related Works}

\textbf{Weight Norm Control.} Weight-norm control approaches have been commonly used in neural network training through reparameterization, shrinkage or constraints. Early works like weight normalization controlled weight geometry \citep{salimans2016weight}. AdamW then decoupled weight decay from adaptive gradient preconditioning \citep{loshchilov2017decoupled}, while AdamP and SGDP \citep{heo2020adamp} subsequently slowed norm growth by removing the radial component of the update for weights they detect as scale-invariant. Later methods introduced norm targets and per-matrix upper bounds \citep{loshchilov2023weight, franke2024improving}, followed by explicit constraints for LLM pretraining \citep{loshchilov2024ngpt, wen2026fantastic, an2026demystifying, hagele2026improving, xie2026controlled}. 
However, these methods do not describe how fast the norm grows, nor do they analyze what impact the increased weight norm has on the overall algorithm performance.

\textbf{Effective Learning Rate.} The effective learning rate describes how parameter scale modifies the effect of a nominal learning rate. In scale-invariant settings, \citet{van2017l2} derived effective learning rates proportional to $\eta/\|W\|^2$ for SGD and $\eta/\|W\|$ for RMSProp and Adam under their respective scaling assumptions. Later works related weight norms and decay to effective learning-rate schedules \citep{hoffer2018norm, zhang2018three}, set up equivalences between weight decay and learning-rate schedules \citep{li2019exponential}, and studied related concepts such as spherical equilibrium, periodic behavior, and optimizer-dependent angular dynamics \citep{wan2021spherical,lobacheva2021periodic,roburin2022spherical, kosson2023rotational}. In LLM pretraining, the effective learning rate also helps explain why weight decay improves training \citep{dangelo2024why} and why SGD underperforms Adam \citep{glentis2026revisiting}. More recently, \citet{liu2026effective} observed that effective learning rate determines the loss dynamics, while \citet{amin2026does} established stability results for effective learning-rate schedules. However, they do not explore how much the effective learning rate can accumulate over training, or relate this accumulation to exploration capacity. 

\section{Norm Control Via Accumulated Effective Learning Rate}

We begin with formalizing the global perspective on norm control by relating Muon's exploration over training intervals to the accumulated effective learning rate $C_T$. Section~\ref{sec:setting} starts from what exploration means for a weight matrix of an LLM, derives the effective learning rate as the quantity that sets the size of a Muon step relative to the weights, and shows that the sum of effective learning rates over a training interval bounds how far training can explore in that interval. Section~\ref{sec:control} compares how fast this sum grows under the three forms in which Muon is run in practice, without norm control, with decoupled weight decay, and with a fixed norm. Then it further answers whether a larger nominal learning rate can substitute for norm control. Proofs are collected in Appendix~\ref{app:proof}.

\subsection{Exploration and accumulated effective learning rate}\label{sec:setting}

Consider a weight matrix $W_t\in\mathbb{R}^{m\times n}$ of a transformer block updated by Muon \citep{jordan2024muon}, and denote $p=\min\{m, n\}$. Vanilla Muon updates it as
\begin{equation}\label{eq:update}
W_{t+1}=W_t-\eta_tO_t,\qquad O_t=\mathrm{Orthogonalize}(U_{t+1}),
\end{equation}
where $\eta_t\ge 0$ is the nominal learning rate and $U_{t+1}$ is the gradient momentum. Exact orthogonalization replaces each nonzero singular value of $U_{t+1}$ by one, giving
\begin{equation}
 \|O_t\|_F=\sqrt{\operatorname{rank}(U_{t+1})}\le\sqrt p.
 \label{eq:muon-bound}
\end{equation}
Thus a full-rank Muon direction has norm $\sqrt p$, independent of the
gradient magnitude\footnote{Practical implementations use Newton-Schulz iterations \citep{jordan2024muon} and multiply the update by a constant that depends on the shape of the matrix \citep{liu2025muon}. These implementations yield a constant-scale change in the bound of $\|O_t\|_F$ that can absorbed into $\eta_t$ or $\sqrt{p}$. We further note that all theoretical results below bound through $\|O_t\|_F$, so we can use a constant $M$ to bound $\|O_t\|_F$ for general cases.}. 

Write $r_t=\|W_t\|_F$ for the weight norm and assume $r_t>0$ throughout training. We study exploration through the net movement of a weight matrix relative to its scale. Over an interval $[s,T]$, the quantity
\[
    \frac{\|W_T-W_s\|_F}{r_T}
\]
measures how far the weights have moved relative to their final magnitude. 

To connect this net movement to Muon's updates, first compare each step
with the weights on which it acts:
\begin{equation}
 \frac{\|W_{t+1}-W_t\|_F}{r_t}
 =\sqrt{\operatorname{rank}(U_{t+1})}\,\frac{\eta_t}{r_t}
 \le\sqrt p\,\frac{\eta_t}{r_t}.
 \label{eq:relative-step}
\end{equation}
Muon supplies the dimension-dependent update scale, while $\eta_t/r_t$
sets its size relative to the weights. We therefore use $\eta_t/r_t$ as
the effective learning rate relative to the weight norm
\citep{dangelo2024why,liu2026effective}. Norm growth can reduce this rate
even when the nominal learning rate is unchanged.

Over a training interval $[s, T]$, these relative updates accumulate:
\begin{equation}
    \sum_{t=s}^{T-1}
    \frac{\|W_{t+1}-W_t\|_F}{r_t}
    \le
    \sqrt{p}\sum_{t=s}^{T-1}\frac{\eta_t}{r_t}.
    \label{eq:relativepath}
\end{equation}
This leads to the accumulated effective learning rate
\begin{equation}
    C_0=0,
    \qquad
    C_T=\sum_{t=0}^{T-1}\frac{\eta_t}{r_t}.
    \label{eq:CT}
\end{equation}
The sum $C_T$ captures how much relative updating can accumulate along the trajectory. The following proposition connects this quantity to the net displacement that motivated the analysis.

\begin{proposition}[Necessary accumulation for relative displacement]
\label{prop:necessary}
For the updates above and integers $0\le s<T$,
\begin{equation}
    \frac{\|W_T-W_s\|_F}{\max\{r_s,r_T\}}
    \le \sqrt{p}(C_T-C_s).
    \label{eq:displacement}
\end{equation}
In addition, define $F(W)=f(W/\|W\|_F)$ and suppose $f$ is
$L$-Lipschitz on the unit Frobenius sphere with respect
to Frobenius distance. Then
\begin{equation}
    |F(W_T)-F(W_s)|
    \le 2L\sqrt{p}(C_T-C_s).
    \label{eq:loss_change}
\end{equation}
\end{proposition}

The first bound \eqref{eq:displacement} implies that a weight change of at least a fraction $\delta>0$ of the larger endpoint norm requires $C_T-C_s\ge\delta/\sqrt{p}$ for any weight matrix. The second bound \eqref{eq:loss_change} further assumes a \emph{scale-invariant} loss, 
which means the loss depends only on its direction $W/\|W\|_F$ because its output is normalized. These scale-invariant matrices are common in query and key projections of general transformers. Bound \eqref{eq:loss_change} states that under the scale-invariance and Lipschitz assumptions with $L>0$, a loss decrease of at least $\Delta>0$ requires $C_T-C_s\ge\Delta/(2L\sqrt{p})$. Note that the Lipschitz assumption is mild as $L$ is the Lipschitz constant of the loss restricted to the unit sphere, which is finite for every continuously differentiable loss.
With the two bounds in Proposition \ref{prop:necessary}, small accumulation therefore limits both relative movement and, under these additional assumptions, loss reduction. We further note that Proposition \ref{prop:necessary} serves as a necessary condition that a lower-bound threshold is needed for sufficient exploration of the algorithm. It does not claim that larger exploration capacity can achieve lower loss.

Proposition \ref{prop:necessary} holds for every interval $[s, T]$, not only for the whole interval $[0, T]$. This is particularly relevant to later training because a large $C_T$ may come mainly from early steps. Substantial relative weight movement after step $s$ requires sufficient additional accumulation $C_T-C_s$, so continued exploration requires $C_T$ to keep growing. The next section therefore compares norm controls by two metrics: the relative growth  $C_T -C_1$ with respect to $T$, as well as the later stage relative growth $C_{2T}-C_T$ over the window $[T, 2T]$. Note that we use the relative growth of $C_T$ because it is not sensitive to initialization scale or model size, which is not the case if we only consider $C_T$ alone. See more discussion in Appendix~\ref{app:scale}. 

\subsection{Accumulation under three Muon norm controls}\label{sec:control}

Given the metric of $C_T$, we now compare how $C_T$ grows under three forms of norm control for Muon, that is, leaving the norm uncontrolled, applying decoupled weight decay, or constraining the norm explicitly. They correspond to vanilla Muon \citep{jordan2024muon}, MuonW \citep{liu2025muon}, and the norm constraint used by MuonH \citep{wen2026fantastic}, respectively. To isolate their effect, take a constant nominal learning rate $\eta_t=\eta>0$.

\begin{itemize}
\item \textbf{Vanilla Muon} follows~\eqref{eq:update}.
\item \textbf{MuonW} uses
\begin{equation}
 W_{t+1}=(1-\eta\lambda)W_t-\eta O_t,
 \qquad \lambda>0,\quad 0<\eta\lambda<1.
 \label{eq:wd}
\end{equation}
\item For \textbf{fixed-norm control}, let $r_0=R>0$ and
\begin{equation}
 W_{t+1}=R\frac{W_t-\eta D_t}{\|W_t-\eta D_t\|_F},
 \label{eq:fixed}
\end{equation}
where $D_t$ is the update direction before the weights are renormalized to norm $R$. Fixed norm uses $D_t=O_t$, while
\textbf{MuonH} additionally fixes the update 
norm by setting $D_t=R O_t/\|O_t\|_F$. The accumulation result below uses only the fixed weight norm, i.e., $D_t=O_t$. 
\end{itemize}

The next proposition characterizes the accumulated effective learning rate under these three cases.

\begin{proposition}[Accumulation under norm control]\label{prop:controls}
Assume a constant nominal learning rate $\eta_t=\eta>0$, and let $R_\lambda=\max\{r_0,\sqrt p/\lambda\}$. For every integer $T\ge1$:
\begin{enumerate}[label=(\roman*)]
\item Vanilla Muon satisfies
\[
C_T\ge\frac{1}{\sqrt p}\log\Bigl(1+\frac{\eta\sqrt p\,T}{r_0}\Bigr),\qquad C_{2T}-C_T\ge\frac{1}{\sqrt p}\log\frac{r_0+2\eta\sqrt p\,T}{r_0+\eta\sqrt p\,T}.
\]
\item MuonW satisfies
$C_T\ge \eta T/R_\lambda$ and $C_{2T}-C_T\ge \eta T/R_\lambda$.
\item A fixed norm $R$ gives
$C_T=\eta T/R$ and $C_{2T}-C_T=\eta T/R$.
\end{enumerate}
Further, for vanilla Muon, the lower bound expressed in (i) is tight as there exists a function and an initialization such that the worst-case rate is attainable, i.e., $C_T=\Theta(\log T)$ and
$C_{2T}-C_T\rightarrow\log 2/\sqrt p$.
\end{proposition}

Proposition \ref{prop:controls} shows why norm control matters for accumulation. Without it, the norm can grow linearly in $t$, so the effective learning rate can decay like $1/t$ and $C_T$ can grow as slowly as $\log T$. Over a later window $[T,2T]$, the bound in (i) tends to $(\log2)/\sqrt p$ as $T$ grows, so the accumulation of vanilla Muon over later windows can stay bounded. In comparison, weight decay and a fixed norm keep the norm at most $R_\lambda$ or $R$, so the effective learning rate stays at least $\eta/R_\lambda$ or $\eta/R$, and $C_T$ and $C_{2T}-C_T$ grow at least linearly. By Proposition~\ref{prop:necessary}, bounded accumulation over later windows limits how far the weights can move late in training and, under the scale-invariance and Lipschitz assumptions, how much the loss can still decrease. This offers an explanation for vanilla Muon's slower progress in training \citep{liu2025muon} and for the improvements obtained with weight decay or fixed-norm constraints.

Figure~\ref{fig:toy_accumulation} shows the three forms on the example (see Appendix~\ref{app:example}) that shows the bounds of Proposition~\ref{prop:controls}(i) are tight: $F(W)=\exp(\tr W)$ from $W_0=-4I_2$, where every Muon direction is $O_t=I_2$, so each update $-\eta O_t$ points away from the origin and the norm grows as fast as possible.
Vanilla Muon has $r_t=\sqrt{2}(4+\eta t)$ and $C_T=\Theta(\log T)$, whereas MuonW and fixed-norm control yield $C_T=\Theta(T)$. The distinction is clearer over later windows: $C_{2T}-C_T$ approaches $\log 2/\sqrt{2}$ for vanilla Muon but grows linearly with $T$ under the other two controls.

\begin{figure*}[t]
  \centering
  \includegraphics[width=\textwidth]{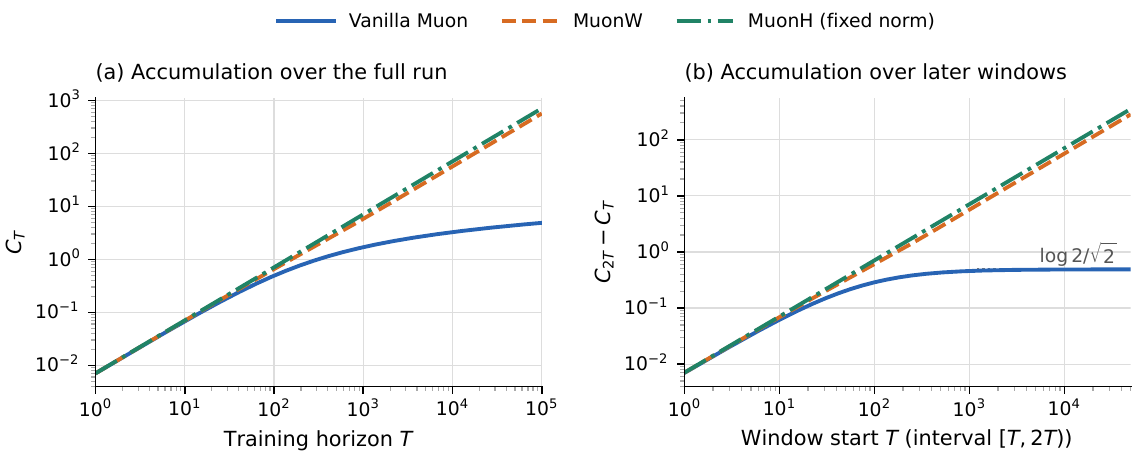}
  \caption{An attainable gap in accumulation growth. Exact Muon is run on $F(W)=\exp(\operatorname{tr}W)$ from $W_0=-4I_2$ for $10^5$ steps, with $\eta=0.04$ and momentum coefficient $\beta=0.95$. MuonW uses $\lambda=0.2$, and MuonH uses $R=4\sqrt2$. All methods have the same orthogonalized direction $I_2$, which MuonH rescales to $D_t=4I_2$. Left: total accumulation. Right: accumulation over $[T,2T)$; the dotted line is $\log 2/\sqrt2$.} 
  \label{fig:toy_accumulation}
\end{figure*}

The example in Proposition \ref{prop:controls} shows that the difference in worst-case growth orders can occur in an actual Muon run.
In practical LLM pretraining, we also observe such numerical differences in accumulation within the same training budget, though the gap is smaller. (See Figure~\ref{fig:CT_horizon}). 
 
\begin{figure*}[t]
  \centering
  \includegraphics[width=\textwidth]{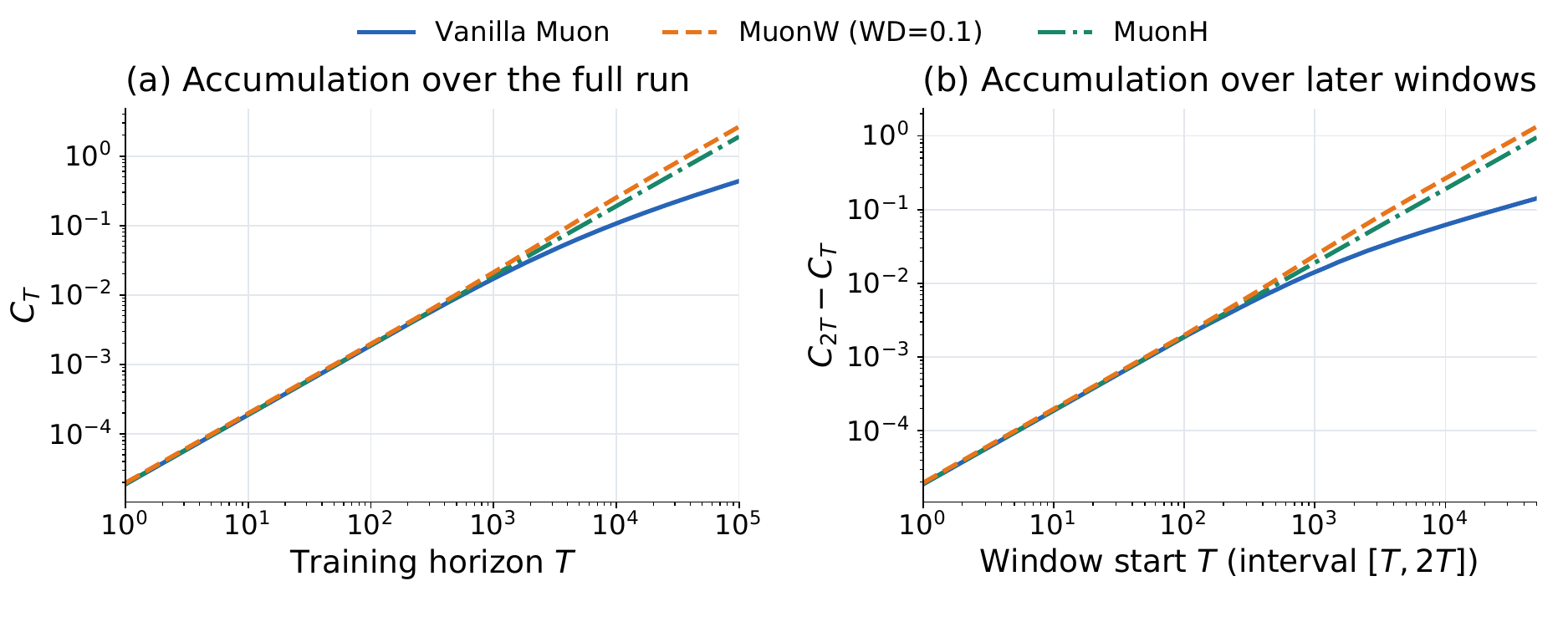}
  \caption{Accumulate effective learning rate $C_T$ vs training horizon $T$. Results from training a 160M GPT-style model on FineWeb for 100,000 steps (about 50B tokens) with a constant learning rate of $0.005$. Left: total accumulation $C_T$. Right: accumulation over $[T,2T)$.}
  \label{fig:CT_horizon}
\end{figure*}

The results so far fix the nominal learning rate. A natural question is whether a larger nominal learning rate can make up for the small accumulation of vanilla Muon. It does not, because a larger $\eta$ also makes the norm grow faster, roughly in proportion to $\eta$, so the effective learning rates after the first step barely increase (Appendix~\ref{app:lr} makes this precise). Figure~\ref{fig:CT_lr} shows this in GPT-160M pretraining: raising the learning rate 16-fold moves $C_{20{,}000}$ of vanilla Muon only from $0.165$ to $0.190$, whereas with weight decay $C_{20{,}000}$ grows with the learning rate. The best MuonW run for each weight decay (stars) reaches $C_{20{,}000}$ between $0.3$ and $0.4$, a level that vanilla Muon does not reach at any learning rate, and a lower perplexity ($21.98$--$22.02$,  versus $22.49$ for the best vanilla Muon run).

\begin{figure*}[t]
  \centering
  \includegraphics[width=1\textwidth]{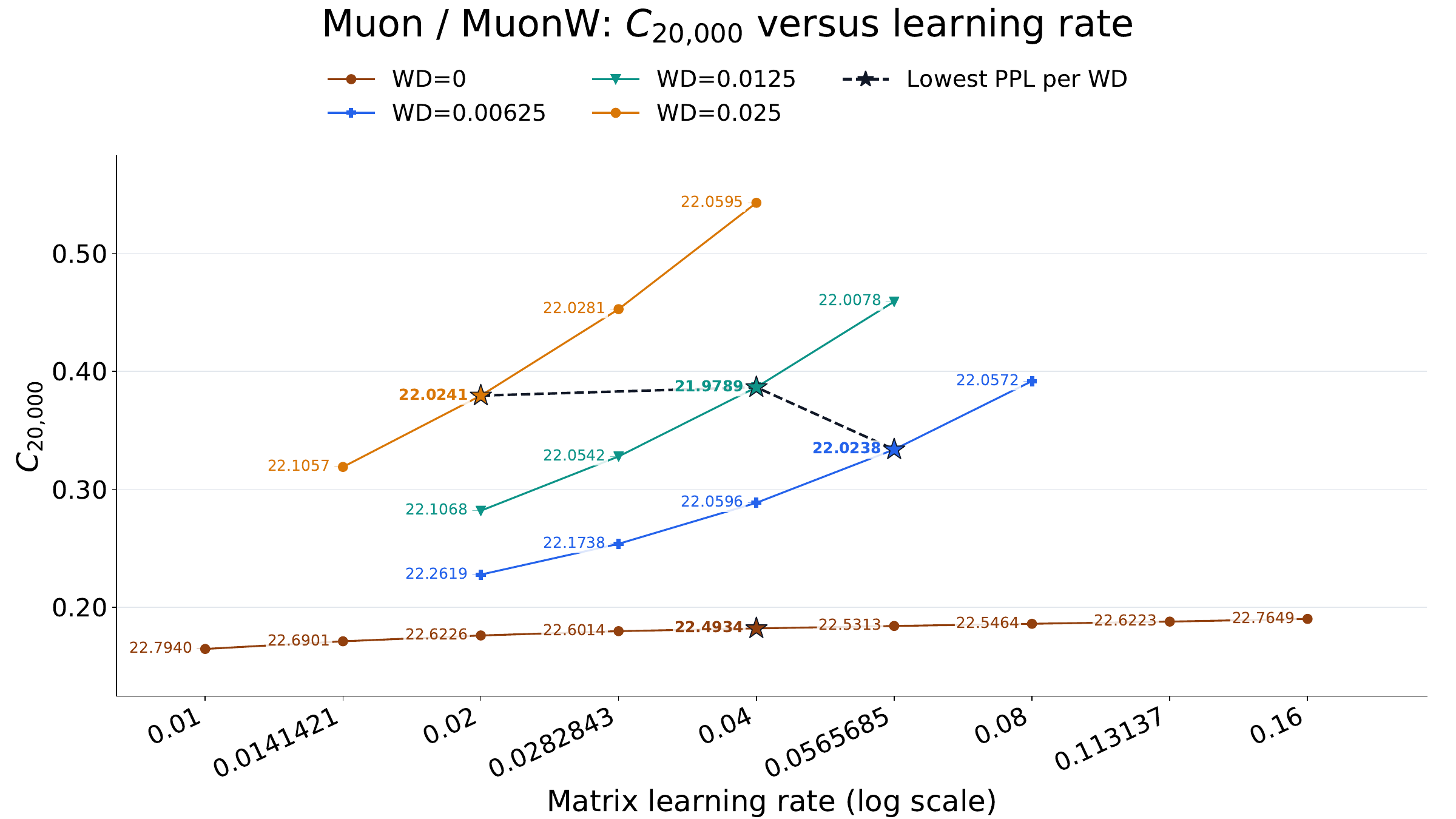}
  \caption{Accumulate effective learning rate $C_T$ vs learning rate scale $\eta$.  Results from training a 160M GPT-style model on FineWeb for 20,000 steps (about 10.5B tokens). Each point represents a learning-rate and weight-decay setting; solid curves connect settings with the same weight decay (WD). WD= 0 is vanilla Muon, and positive WD corresponds to MuonW. Numbers beside the points report perplexity. Stars mark the lowest perplexity among the tested learning rates for each WD; the dashed line connects these stars for MuonW.}
  \label{fig:CT_lr}
\end{figure*}

To summarize, our analysis links norm control to exploration through accumulated effective learning rates. Weight decay and fixed-norm constraints sustain accumulation, while rapid norm growth can limit it and offset a larger nominal learning rate.

\section{MuonGS for LLM Pretraining}\label{sec:gs}

The preceding analysis links the benefit of bounded weight norms to the accumulation they preserve. This suggests relaxing the bounded-norm requirement, since sufficiently slow norm growth can still guarantee faster accumulation than vanilla Muon's logarithmic worst case. For example, at a constant nominal learning rate $\eta>0$, the bound $r_t\le r_0+ct^a$ with $c>0$ and $0\le a<1$ gives
\[
 C_{2T}-C_T\ge\frac{\eta T}{r_0+c(2T)^a}=\Omega(T^{1-a}).
\]
Accumulation can therefore grow over later training windows even when weight norms increase and effective learning rates decrease. This opens up a broader family of Muon variants and brings us back to the question posed in the introduction. \emph{Can this additional flexibility be used to design optimizers that sustain accumulation throughout training and achieve better practical performance?}

We address this question with \emph{MuonGS}, which controls norm growth by removing the radial component of Muon's update. We analyze its accumulation and optimization properties, then evaluate its performance in LLM pretraining.

\subsection{Removing the radial component of the Muon update}
\label{sec:gsupdate}

To construct an update that limits weight-norm growth, we first examine how a vanilla Muon step changes the squared weight norm.
\begin{equation}
 \|W_t-\eta_tO_t\|_F^2
 =r_t^2-2\eta_t\langle W_t,O_t\rangle_F
   +\eta_t^2\|O_t\|_F^2.
 \label{eq:rawnorm}
\end{equation}
The term $-2\eta_t\langle W_t,O_t\rangle_F$ is linear in the step size and increases the squared weight norm when $\langle W_t,O_t\rangle_F<0$. Muon controls the size of $O_t$ through matrix orthogonalization, but the update may still point partly along $W_t$. This radial component is responsible for the first-order norm change.

\begingroup
We therefore seek a modified direction that eliminates the first-order norm change while remaining as close as possible to Muon's proposed update. Giving priority to the norm-control requirement leads to the following simple bilevel formulation for fixed $W_t\ne0$ and $O_t$
\[
 \begin{aligned}
 \min_{D\in\R^{m\times n}}\quad &\frac12\norm{D-O_t}^2\\
 \text{subject to}\quad
 &D\in\operatorname*{arg\,min}_{Z\in\R^{m\times n}}
       \frac12\ip{W_t}{Z}^{\,2}.
 \end{aligned}
\]
The lower level minimizes the squared first-order norm change, so its minimizers are precisely the directions orthogonal to $W_t$. The upper level selects the closest such direction to $O_t$. The unique solution is the Gram--Schmidt (GS) projection onto the tangent space of the sphere through $W_t$. We call the resulting method \emph{MuonGS}, with the update
\endgroup
\begin{equation}
 \widetilde O_t=O_t-
 \frac{\langle W_t,O_t\rangle_F}{r_t^2}W_t,
 \qquad W_{t+1}=W_t-\eta_t\widetilde O_t.
 \label{eq:gs}
\end{equation}
The projection gives the solution in closed form, so no additional iterative optimization is required. Since $O_t$ is treated as given, it also applies when the Muon direction incorporates momentum and approximate matrix orthogonalization.

The projection gives $\langle W_t,\widetilde O_t\rangle_F=0$ and $\|\widetilde O_t\|_F\le \sqrt{p}$. Substituting into~\eqref{eq:rawnorm} leaves
\begin{equation}
 r_{t+1}^2=r_t^2+\eta_t^2\|\widetilde O_t\|_F^2.
 \label{eq:gsidentity}
\end{equation}
Thus the norm may increase, but its squared increment is only quadratic in the step size. AdamP applies the same tangent-space projection to the Adam update, removing its radial component to suppress momentum-induced norm growth on weights it detects as scale-invariant, i.e., when the cosine similarity between the gradient and the weight falls below a small threshold ~\mbox{\citep{heo2020adamp}}. MuonGS applies this projection directly to Muon's matrix updates, without a scale-invariance test. More importantly, beyond this algorithmic distinction, our analysis links the resulting norm-growth control to exploration over entire training trajectories. As shown next, for bounded update directions and a constant nominal learning rate, MuonGS guarantees $r_t=O(\sqrt{t})$, $C_T=\Omega(\sqrt{T})$, and $C_{2T}-C_T=\Omega(\sqrt{T})$. These guarantees hold without assuming a scale-invariant objective.

\begin{algorithm}[t]
\caption{MuonGS for one weight matrix (the projection is line~4).}
\label{alg:muongs}
\begin{algorithmic}[1]
\Require learning rate $\eta_t$, momentum $\beta$, initial weights $W_0$, $U_0=0$
\For{$t=0,1,2,\dots$}
  \State $U_{t+1}\gets\beta U_t+\nabla f(W_t)$ \Comment{gradient momentum}
  \State $O_t\gets\mathrm{Orthogonalize}(U_{t+1})$ \Comment{Newton--Schulz, as in Muon}
  \State $\widetilde O_t\gets O_t-\dfrac{\langle W_t, O_t\rangle}{\|W_t\|_F^2}\,W_t$ \Comment{remove the component along $W_t$}
  \State $W_{t+1}\gets W_t-\eta_t\widetilde O_t$
\EndFor
\end{algorithmic}
\end{algorithm}

\subsection{Norm growth and accumulation}
We now quantify how the projection's control of norm growth sustains accumulation over training. Starting from $r_0>0$ and using the bound $\|O_t\|_F\le\sqrt{p}$, the following theorem controls weight-norm growth under arbitrary learning-rate schedules. For a constant learning rate, it then gives lower bounds on both $C_T$ and $C_{2T}-C_T$, showing that accumulation is sustained during later training as well as over the full trajectory.

\begin{theorem}[MuonGS growth and accumulation]\label{thm:gs}
Under \eqref{eq:gs}, all iterates remain nonzero and
\begin{equation}\label{eq:gs_radius}
r_T^2
=r_0^2+\sum_{t=0}^{T-1}\eta_t^2\norm{\widetilde O_t}^{\,2}
\le r_0^2+p\sum_{t=0}^{T-1}\eta_t^2.
\end{equation}
For $\eta_t=\eta>0$,
\begin{align}
 C_T&\ge\frac{2}{p\eta}
 \left(\sqrt{r_0^2+p\eta^2T}-r_0\right)
 =\Omega(\sqrt T),\label{eq:gsacc}\\
 C_{2T}-C_T&\ge\frac{\eta T}{\sqrt{r_0^2+2p\eta^2T}}
 =\Omega(\sqrt T).\label{eq:gswindow}
\end{align}
\end{theorem}

MuonGS therefore limits the norm to $O(\sqrt t)$ and guarantees at least square-root accumulation, including over later windows. It improves the worst-case guarantee of vanilla Muon without enforcing a fixed radius.
Thus the accumulation preserved by MuonGS is available during later training, rather than coming only from early updates. Table \ref{tab:orders} summarizes the comparison with vanilla Muon and bounded-norm controls.

\begin{table}[htbp]
\centering\small
\caption{Norm and accumulation guarantees at a constant positive learning rate.}\label{tab:orders}
\begin{tabular}{lccc}
\toprule
Control & Norm bound & $C_T$ & $C_{2T}-C_T$\\
\midrule
Vanilla Muon & $O(t)$ & $\Omega(\log T)$ & $\Omega(1)$\\
Weight decay & $O(1)$ & $\Omega(T)$ & $\Omega(T)$\\
Fixed norm & $R$ & $\eta T/R$ & $\eta T/R$\\
GS & $O(\sqrt t)$ & $\Omega(\sqrt T)$ & $\Omega(\sqrt T)$\\
\bottomrule
\end{tabular}
\end{table}

The preceding bounds describe the relative updating available along a trajectory. To connect this accumulation to optimization, we now consider a fixed scale-invariant matrix objective
\begin{equation}
 F(W)=f(W/\|W\|_F),\qquad O_t=\operatorname{Orthogonalize}(\nabla F(W_t)),
 \label{eq:idealized}
\end{equation}
where $f$ is twice continuously differentiable on the unit Frobenius sphere.

\begin{proposition}[Stationarity for exact current-gradient MuonGS]
\label{prop:stationarity}
Under~\eqref{eq:idealized}, MuonGS with $r_0>0$ and a constant learning rate $\eta>0$ satisfies, for every even $T\ge2$,
\begin{equation}
 \min_{T/2\le t<T}\|\operatorname{grad}f(W_t/r_t)\|_F\le\frac{K}{\sqrt T},
 \label{eq:stationarity}
\end{equation}
where $\operatorname{grad}f$ is the sphere gradient and $K<\infty$ is independent of $T$.
\end{proposition}

Here the current-gradient direction provides descent alignment, while GS controls the weight norm and the resulting effective learning rates. Together they guarantee an iterate with a small sphere-gradient norm in the second half of training, even though the weight norm can keep growing.

\section{Experiments}

In this section we evaluate the \ourname{} optimizer on the task of LLM pre-training. Since the results from \cite{wen2026fantastic} indicate that the MuonH outperforms MuonW in typical pre-training settings, we focus our main comparison to MuonH. Note that both MuonGS and MuonH share the favorable property of not needing to tune a weight decay coefficient as opposed to MuonW. We include some additional results comparing with Muon and MuonW in Appendix \ref{app:muon_results}. Below we present our main results using dense GPT-style models, but include preliminary results on MoE model training in Section \ref{sec:moe}.

\subsection{Experimental settings}

\textbf{Architecture.} We use the modernized GPT-style architecture from the Modded-NanoGPT Optimization Benchmark \citep{moddednanogpt_optimbench_2026} at two scales, approximately 160M and 400M parameters. The smaller model, which corresponds to the NanoGPT model used in \cite{moddednanogpt_optimbench_2026}, has 12 layers, hidden width 768, and 6 attention heads; the other is a scaled-up version and has 24 layers, hidden width 1024, and 8 heads, with head dimension 128 for both sizes. Both models use RMSNorm, rotary positional embeddings, QK-Norm, squared-ReLU MLPs, and untied input and output embeddings. 

\textbf{Pre-training setup.} We pre-train on the 100B-token subset of the FineWeb dataset \citep{penedo2024finewebdatasetsdecantingweb} using a sequence length of 1024 and a global batch of 512 sequences. Our experiments span approximately 3--25B training tokens for the 160M model and 8--63B for the 400M model, corresponding to roughly 1--8 Chinchilla budgets \citep{hoffmann2022training} at both scales. We use BF16 mixed-precision training.

\textbf{Optimizers.} We compare \ourname{} as described in Algorithm \ref{alg:muongs} with MuonH and Muon/MuonW. All methods use Nesterov momentum $0.95$ for hidden matrix weights. AdamW is used for token embeddings, output-head weights, and all 1D parameters, with $(\beta_1,\beta_2)=(0.8,0.95)$, $\epsilon=10^{-10}$, weight decay $10^{-3}$, and respective max learning rates $0.7$, $0.004$, and $0.015$. Matrix and auxiliary learning rates decay linearly to zero without warmup after constant phases covering $15\%$ and $30\%$ of training, respectively.

\textbf{Tuning and metrics.}  We tune learning rates on a  multiplicative grid with ratio $\sqrt{2}$ similar to \cite{wen2026fantastic}, separately for all optimizers, model sizes and training durations. For MuonW we additionally tune the weight decay as shown in Appendix \ref{app:muon_results}. We use as metric the cross-entropy loss $\mathcal{L}_{\mathrm{val}}$ over about 10M tokens from the FineWeb validation set, equivalently expressed as perplexity $\mathrm{PPL}=\exp(\mathcal{L}_{\mathrm{val}})$.

\subsection{Benchmarking Results}

\textbf{Learning Rate Sweeps.} From Figure \ref{fig:main_sweep} we observe that, for extended training horizons (overtraining), \ourname{} achieves lower perplexity than MuonH. We note that in modern LLM pre-training extending to beyond one Chinchilla training tokens, i.e.,. more than 20 tokens per model parameter, is common practice \citep{grattafiori2024llama, yang2025qwen3}. Therefore, \ourname's favorable scaling with training horizon is particularly relevant to modern LLM pre-training.

\begin{figure}
    \centering
    \includegraphics[width=1.0\textwidth]{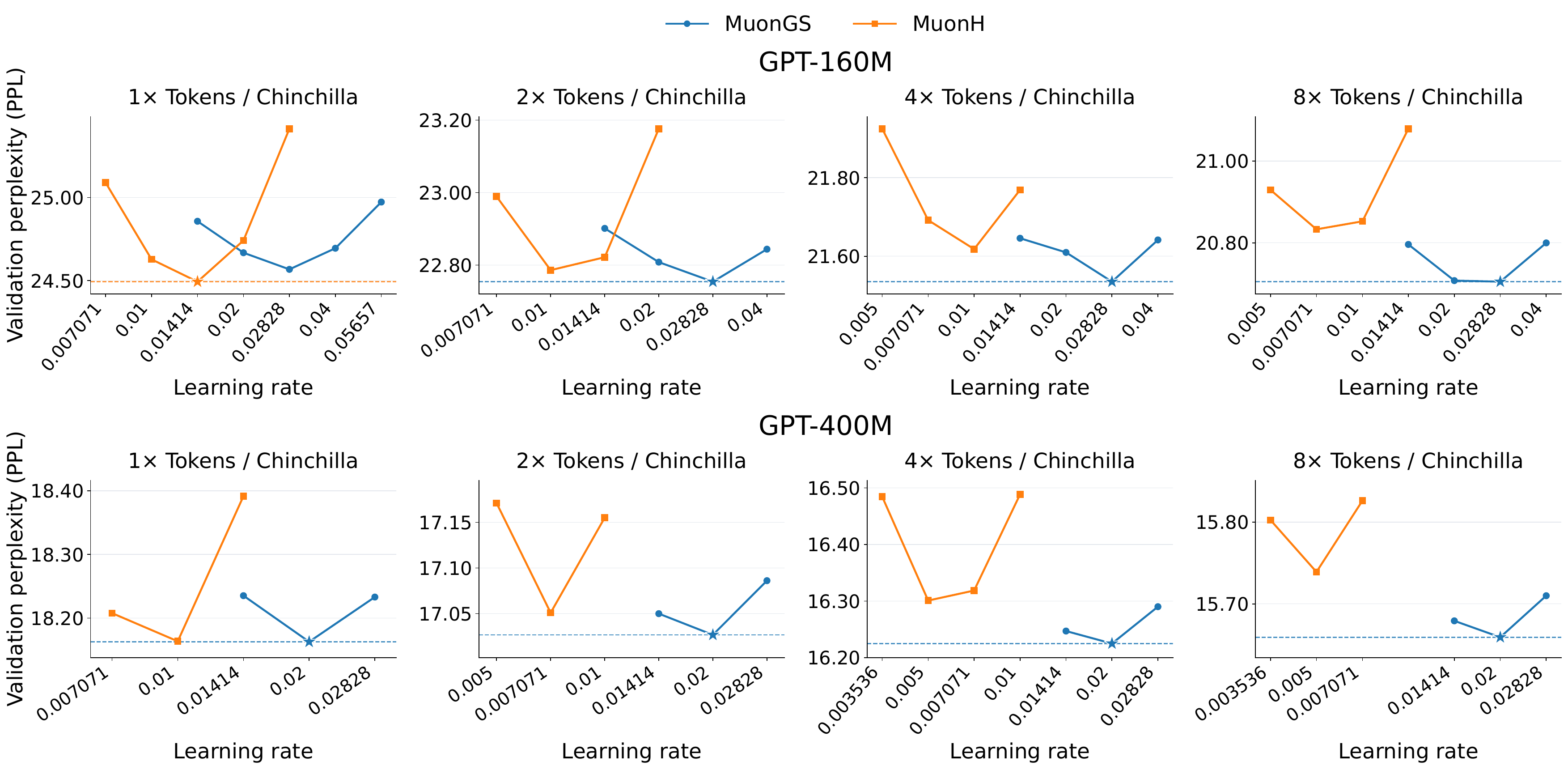}
    \caption{Validation perplexity versus learning rate for MuonGS and MuonH across the two model sizes and Chinchilla ration $1{\times}$--$8{\times}$. Stars and dashed horizontal lines indicate the best observed perplexity in each panel.}
    \label{fig:main_sweep}
\end{figure}

\textbf{Speedup Estimation.} Using the best results for each algorithm from the sweeps above, we estimate the training-budget speedup of \ourname{} over MuonH following the method used in \cite{wen2025fantastic}. Specifically, for each model size $N$, we fit MuonH's validation loss as a function of training-token budget $D$ by nonlinear least squares:
\[
\widehat{L}_{N}(D)=\alpha_N D^{-B_N}+\beta_N.
\]
For a \ourname{} run achieving loss $L_{\mathrm{GS}}$ at budget $D_{\mathrm{GS}}$, we solve
$\widehat{L}_{N}(D_{\mathrm{H}})=L_{\mathrm{GS}}$ and report $D_{\mathrm{H}}/D_{\mathrm{GS}}$ as the estimated speedup, where $D_{\mathrm{H}}$ is the estimated training-token budget MuonH would require to match the validation loss achieved by \ourname{}. 

From Figure \ref{fig:muongs-speedup} we observe that \ourname{} achieves a token-equivalent speedup over MuonH up to about 7-15\% for longer training horizons.

\begin{figure}
    \centering
    \includegraphics[width=\textwidth]{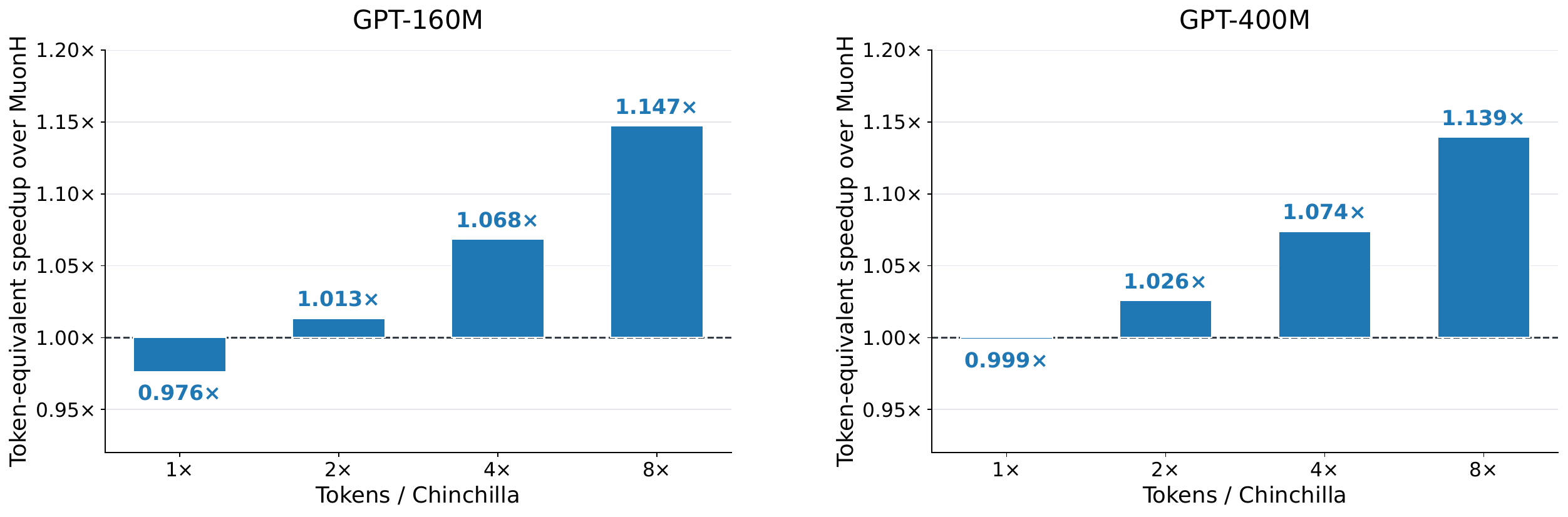}
    \caption{Estimated token-equivalent speedup of MuonGS over MuonH for GPT-160M (left) and GPT-400M (right) across Chinclilla ratios $1{\times}$--$8{\times}$.}
    \label{fig:muongs-speedup}
\end{figure}

\subsection{MoE experiments} 
\label{sec:moe}

\begin{figure}
    \centering
    \includegraphics[width=\textwidth]{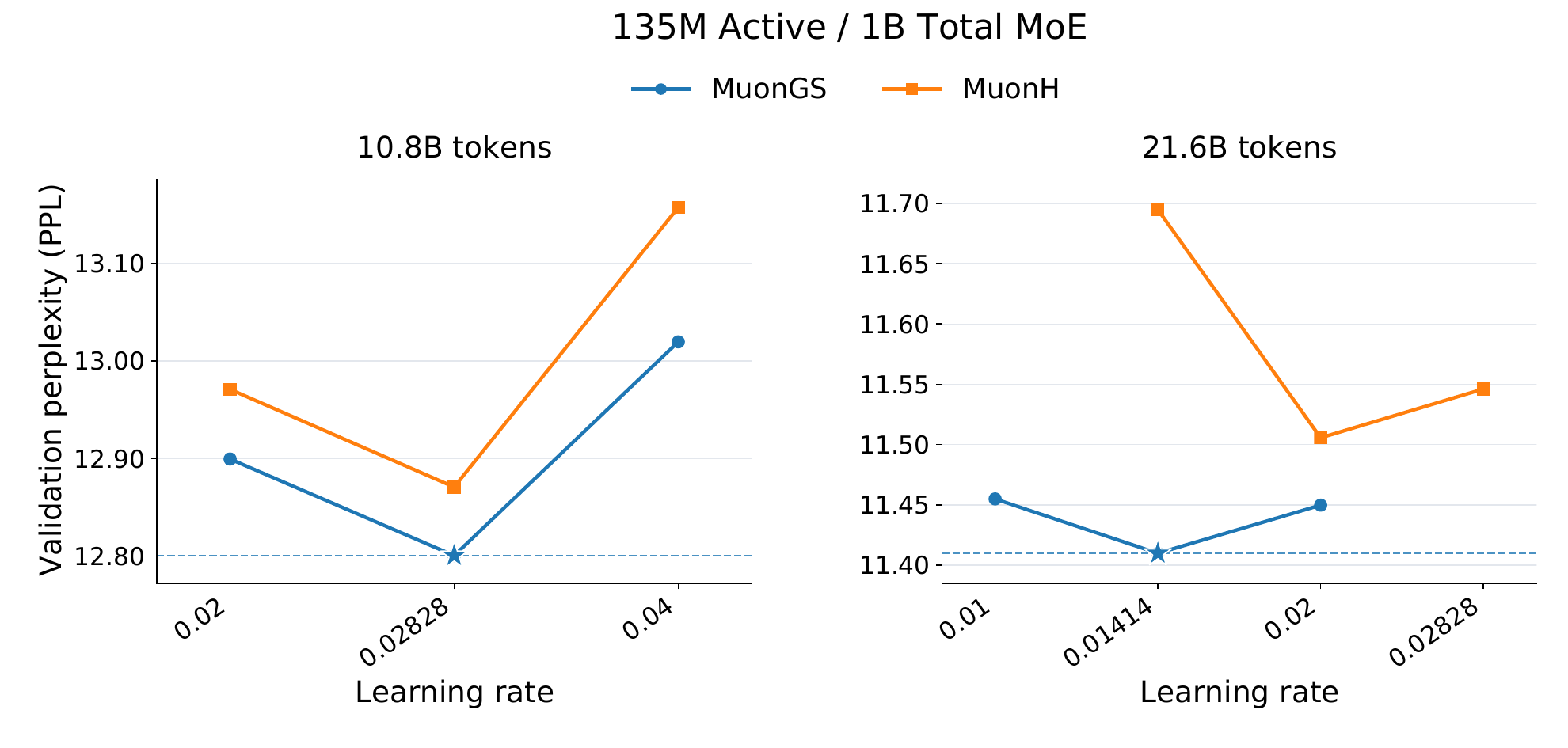}
    \caption{Validation perplexity versus learning rate for \ourname{} and MuonH for an 135M active / 1B total MoE model for two different training durations. Stars and dashed horizontal lines indicate the best observed perplexity in each panel.}
    \label{fig:moe-ppl}
\end{figure}

\vspace{-10pt}

In this section we include preliminary results comparing \ourname{} to MuonH on a modern Mixture-of-Experts (MoE) architecture. The model incorporates grouped-query attention \citep{ainslie2023gqa}, gated attention \citep{qiu2026gated} and attention residuals (AttnRes) \citep{team2026attention}. It has 32 layers, hidden width 512, and 64 routed experts with top-7 routing plus one shared expert per layer, yielding approximately 135M active non-embedding parameters and 1B total parameters. We train on Ultra-FineWeb \citep{wang2025ultra} for 10.8B and 21.6B tokens (i.e., 80 and 160 times the active parameters), using a batch size of 1,024 sequences of length 4,096. Learning rates are tuned on a $\sqrt{2}$-spaced grid, with 2\% warmup followed by linear decay to zero. Both \ourname{} and MuonH use momentum $0.95$ with Nesterov acceleration. 

From Figure \ref{fig:moe-ppl}, we observe that \ourname{} achieves lower best validation perplexity than MuonH at both training budgets, making it a promising method for modern MoE optimization. 

\section{Conclusion}
 
We reexamined the role of weight norm control for Muon over the entire training trajectory instead of a single step. In particular, we use the accumulated effective learning rate to measure how far the weights can move relative to their norm over any training interval, and shows that continued exploration requires it to keep growing. We further show that vanilla Muon can only guarantee the accumulation to grow logarithmically, while weight norm control approaches such as weight decay and fixed-norm control guarantee a linear accumulation growth, which explains the role of weight norm control. Based on this finding, we relax the weight norm control to a larger range of approaches. In particular, we propose MuonGS by simply removing the radial component of the Muon update. Theoretically, we show that MuonGS guarantees an accumulation of at least $\sqrt{T}$, sufficient for continued exploration. Numerically, MuonGS outperforms MuonW and MuonH at long training horizons, providing a speedup of up to $15\%$ at matched validation loss. Our accumulation guarantees are worst-case bounds, and the accumulation guarantee provides a necessary but not sufficient condition for continued exploration. Which rate of norm growth works best for a given training budget remains an open question.


\newpage

\bibliographystyle{plainnat}
\bibliography{reference.bib}


\newpage

\appendix

\section{Additional Experiments}\label{app:experiments}

\subsection{Experiments with Muon/MuonW}
\label{app:muon_results}

In Figure \ref{fig:muon_comparison_20k} we include a comparison of MuonH and the proposed \ourname{} over Muon and MuonW (with varying weight decay). We observe that Muon, not using any type of weight norm control, performs noticeably worse than the other baselines. \ourname{} obtains the best perplexity, followed by MuonH and MuonW.

\begin{figure}[H]
    \centering
    \includegraphics[width=\linewidth]{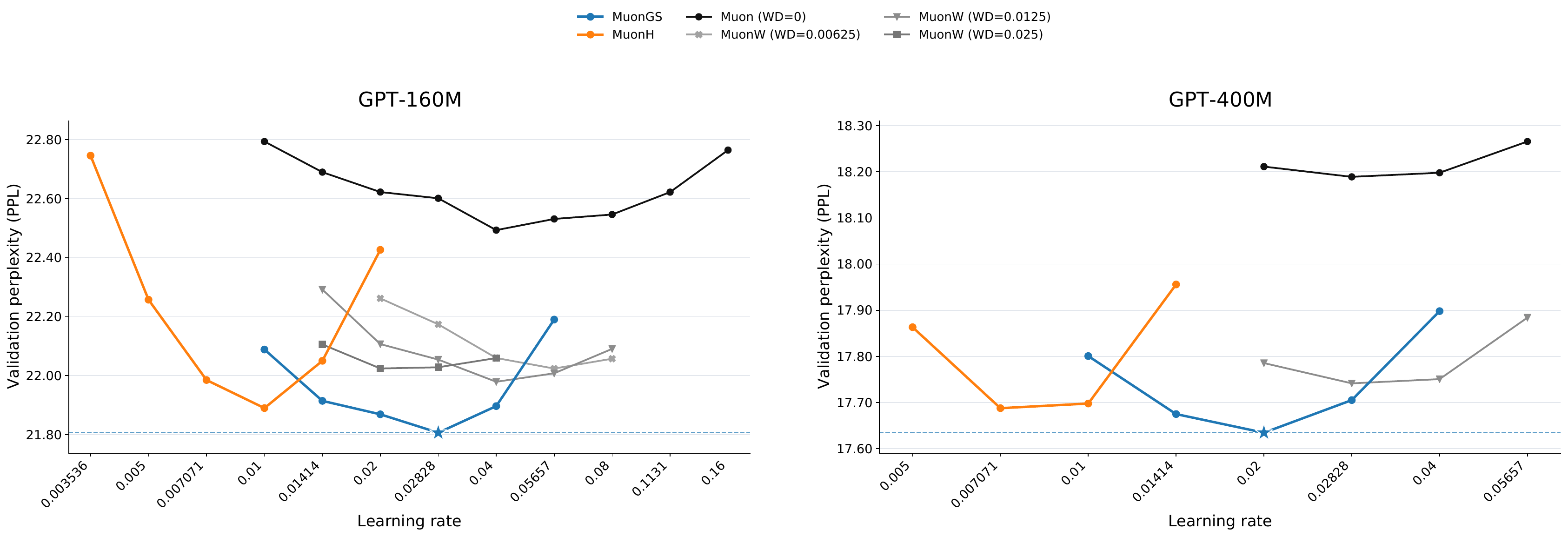}
    \caption{Validation perplexity vs learning rate for \ourname{}, MuonH, Muon and MuonW with varying weight decay. Both models are trained for 20k steps (about 10.5B tokens) on FineWeb.}
    \label{fig:muon_comparison_20k}
\end{figure}

We also compare the evolution of the norm of the combined matrix weights, $\|w_t\|_2$, during training. As shown in Figure \ref{fig:weight_norms}, MuonH maintains a constant norm, while Muon without weight decay exhibits considerably larger norm growth. MuonW's norm initially increases and subsequently declines, whereas MuonGS exhibits more moderate, gradual norm growth without weight decay or an explicit fixed-norm constraint.
   
\begin{figure}[H]
    \centering
    \includegraphics[width=\linewidth]{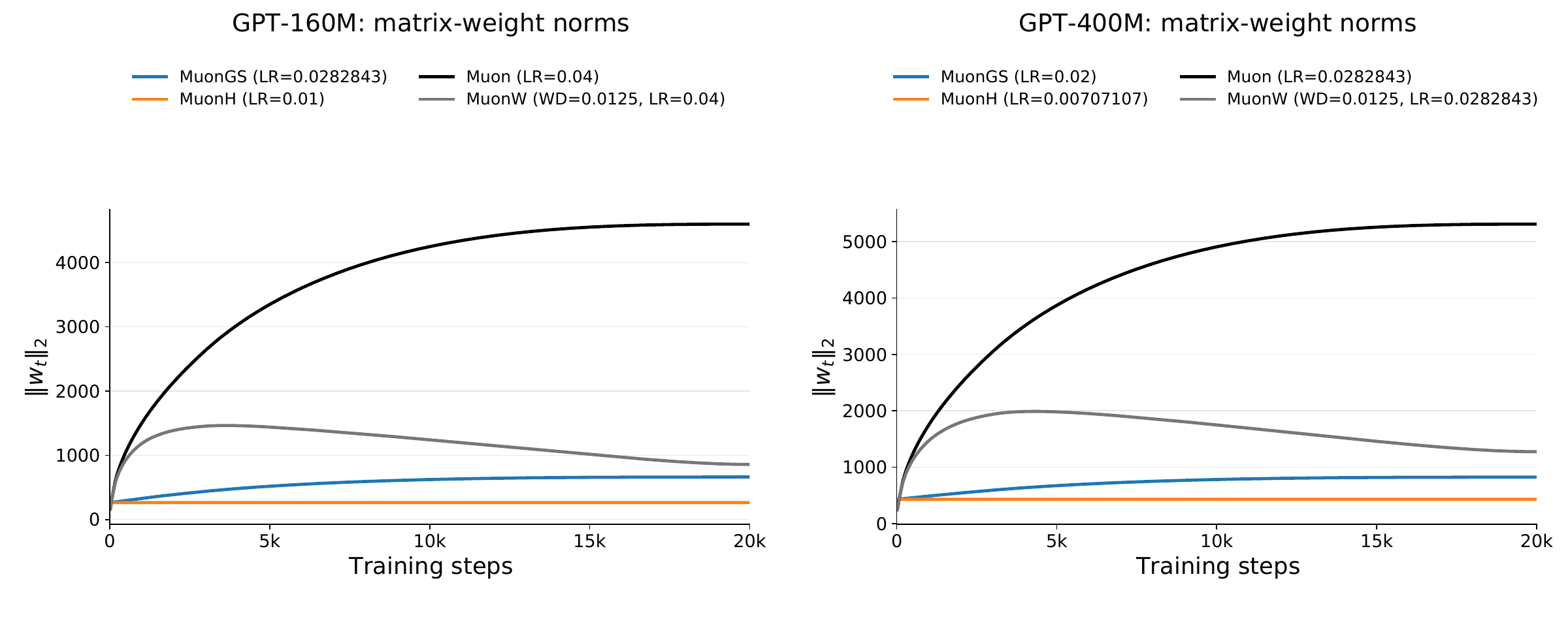}
    \caption{Weight-norm evolution for the 160M (left) and 400M (right) models over 20k training steps, using each method's best learning rate. Here, $w_t$ concatenates the matrix weights updated by Muon variants, excluding auxiliary AdamW parameters. MuonW uses weight decay $0.0125$ which was found to be the best among the ones tested as shown \ref{fig:muon_comparison_20k}}.
    \label{fig:weight_norms}
\end{figure}

\section{Proofs and Examples}\label{app:proof}

\subsection{Proof of Proposition~\ref{prop:necessary}}
\begin{proof}
For nonzero matrices $X,Y$ let
\[
d(X,Y)=\frac{\|X-Y\|_F}{\max\{\|X\|_F,\|Y\|_F\}}.
\]
We first show that $d(X,Y)\le d(X,Z)+d(Z,Y)$. Write $x=\|X\|_F$, $y=\|Y\|_F$, $z=\|Z\|_F$, and assume $x\ge y$. If $z\le x$, both denominators on the right are at most $x$, and the claim follows from the triangle inequality for $\|\cdot\|_F$. If $z>x$, use the inversion $\iota(X)=X/\|X\|_F^2$, which satisfies $\|\iota(X)-\iota(Y)\|_F=\|X-Y\|_F/(xy)$. Then
\[
d(X,Y)=y\|\iota(X)-\iota(Y)\|_F,\quad d(X,Z)=x\|\iota(X)-\iota(Z)\|_F,\quad d(Z,Y)=y\|\iota(Z)-\iota(Y)\|_F,
\]
and the claim follows from the triangle inequality for $\iota$ and $y\le x$.

Applying this inequality along the iterates and then~\eqref{eq:relative-step},
\[
\frac{\|W_T-W_s\|_F}{\max\{r_s,r_T\}}\le\sum_{t=s}^{T-1}\frac{\|W_{t+1}-W_t\|_F}{\max\{r_t,r_{t+1}\}}\le\sum_{t=s}^{T-1}\frac{\eta_t\|O_t\|_F}{r_t}\le\sqrt p\,(C_T-C_s),
\]
which is~\eqref{eq:displacement}. For~\eqref{eq:loss_change}, let $x\ge y$ again; since $x-y\le\|X-Y\|_F$,
\[
\Bigl\|\frac Xx-\frac Yy\Bigr\|_F\le\frac{\|X-Y\|_F}{x}+\|Y\|_F\Bigl(\frac1y-\frac1x\Bigr)=\frac{\|X-Y\|_F+x-y}{x}\le2\,d(X,Y).
\]
Hence $|F(W_T)-F(W_s)|\le L\|W_T/r_T-W_s/r_s\|_F\le2L\,d(W_T,W_s)$, and~\eqref{eq:displacement} gives~\eqref{eq:loss_change}.
\end{proof}

\subsection{Proof of Proposition~\ref{prop:controls}}\label{app:controls}
\begin{proof}

\emph{Vanilla Muon.} Fix integers $0\le s < u$. The triangle inequality gives $r_{t+1}\le r_t+\eta_t\sqrt p$, so at a constant learning rate $r_t\le r_0+\eta\sqrt p\,t$. The function $x\mapsto\eta/(r_0+\eta\sqrt p\,x)$ is decreasing, so its value at an integer $t$ is at least its integral over $[t,t+1]$. Summing over $s\le t<u$ gives
\[
C_u-C_s=\sum_{t=s}^{u-1}\frac{\eta}{r_t}\ge\int_s^u\frac{\eta\,dx}{r_0+\eta\sqrt p\,x}=\frac{1}{\sqrt p}\log\frac{r_0+\eta\sqrt p\,u}{r_0+\eta\sqrt p\,s}.
\]
Taking $(s,u)=(0,T)$ and $(s,u)=(T,2T)$ give (i).

\emph{MuonW.} For $0\le\eta_t\lambda\le1$, $r_{t+1}\le(1-\eta_t\lambda)r_t+\eta_t\sqrt p$. Since $\sqrt p\le\lambda R_\lambda$, the bound $r_t\le R_\lambda$ implies $r_{t+1}\le R_\lambda-\eta_t(\lambda R_\lambda-\sqrt p)\le R_\lambda$. Therefore by induction $r_t\le R_\lambda$ for all $t$, so
\[
C_u-C_s=\sum_{t=s}^{u-1}\frac{\eta_t}{r_t}\ge\frac{1}{R_\lambda}\sum_{t=s}^{u-1}\eta_t.
\]
At a constant learning rate, the steps $0\le t<T$ and $T\le t<2T$ each contribute $\eta T$ to the last sum, which gives (ii).

\emph{Fixed norm.} Renormalization keeps $r_t=R$ as long as the weight before renormalization is nonzero, so $C_u-C_s=R^{-1}\sum_{t=s}^{u-1}\eta_t$, which gives (iii) in the same way.

The tightness of (i) is shown by the example of Appendix~\ref{app:example}.
\end{proof}

\subsection{An exact Muon example}\label{app:example}

Let $F(W)=\exp(\tr W)$ on $\mathbb{R}^{p\times p}$ and $W_0=-bI_p$ with $b>0$. The gradient $\nabla F(W)=\exp(\tr W)I_p$ is a positive multiple of $I_p$, at every $W$, and so is the momentum (also with Nesterov's variant). Exact orthogonalization therefore gives $O_t=I_p$ with $\|O_t\|_F=\sqrt{p}$ at every step. For every method below, we choose a constant nominal learning rate $\eta_t=\eta>0$.

\paragraph{Vanilla Muon.} By induction,
\begin{equation}\label{eq:exraw}
W_t=-(b+\eta t)I_p,\qquad r_t=\sqrt p\,(b+\eta t),\qquad\frac{\eta}{r_t}=\frac{1}{\sqrt p}\,\frac{\eta}{b+\eta t}.
\end{equation}
The applied update $-\eta O_t$ is a positive multiple of $W_t$, so every step points directly away from the origin. In particular, $-\langle W_t,O_t\rangle_F/r_t=\sqrt p$, and the norm bound $r_{t+1}\le r_t+\eta\sqrt p$ in Appendix~\ref{app:controls} holds with equality. 

We now calculate $C_u-C_s$:
\[
C_u-C_s
=\sum_{t=s}^{u-1}\frac{\eta}{r_t}
=\frac{1}{\sqrt p}\sum_{t=s}^{u-1}\frac{\eta}{b+\eta t}.
\]
Since the function $x\mapsto 1/(b+x)$ is decreasing, we have
\[
0\le
\frac{\eta}{b+\eta t}
-\int_{\eta t}^{\eta(t+1)}\frac{dx}{b+x}
\le
\eta\left(\frac{1}{b+\eta t}
-\frac{1}{b+\eta(t+1)}\right).
\]
Summing over $s\le t<u$ and dividing by $\sqrt p$ gives
\begin{equation}\label{eq:excompare}
0\le C_u-C_s-\frac{1}{\sqrt p}\log\frac{b+\eta u}{b+\eta s}
\le\frac{\eta}{\sqrt p}
\left(\frac{1}{b+\eta s}-\frac{1}{b+\eta u}\right)
\le\frac{\eta}{b\sqrt p}.
\end{equation}
Since $r_0=b\sqrt p$, the logarithms for $(s,u)=(0,T)$ and $(T,2T)$ are exactly the two bounds of Proposition~\ref{prop:controls}(i). So $C_T$ exceeds the first bound by at most $\eta/(b\sqrt p)=C_1$, $C_{2T}-C_T$ exceeds the second by at most $\eta/(\sqrt p\,(b+\eta T))\to0$, and, as $T\to\infty$,
\begin{equation}\label{eq:exconst}
C_T=\frac{\log T}{\sqrt p}+O(1),\qquad C_{2T}-C_T\longrightarrow\frac{\log2}{\sqrt p}.
\end{equation}

\paragraph{MuonW and fixed norm.} With weight decay and $q=1-\eta\lambda\in(0,1)$, the update $W_{t+1}=qW_t-\eta I_p$ gives
\begin{equation}\label{eq:exwd}
W_t=-b_tI_p,\qquad b_t=\frac1\lambda+\Bigl(b-\frac1\lambda\Bigr)q^t,\qquad r_t=\sqrt p\,b_t.
\end{equation}
Since $b_t\to1/\lambda$ geometrically, $\eta/r_t=\eta\lambda/\sqrt p+O(q^t)$. Summing gives $C_T=\eta\lambda T/\sqrt p+O(1)$ and $C_{2T}-C_T=\eta\lambda T/\sqrt p+O(q^T)$.

With a fixed norm $R=r_0$, the weight before renormalization is $-(b+\eta)I_p$ for $D_t=O_t$, and $-b(1+\eta)I_p$ for MuonH, whose rescaled direction is $D_t=RO_t/\|O_t\|_F=bI_p$. Both renormalize to $W_0$, so $W_t=W_0$ and $C_T=C_{2T}-C_T=\eta T/r_0$.

\paragraph{Numerical illustration in Figure~\ref{fig:toy_accumulation}.} In Figure~\ref{fig:toy_accumulation}, we choose $p=2$, $b=4$, $\eta=0.04$, $\lambda=0.2$, momentum coefficient $\beta=0.95$ and $10^5$ steps. The corresponding weight norms are thus
\begin{equation}\label{eq:exnorms}
r_t^{\mathrm{vanilla}}=\sqrt2\,(4+0.04t),\qquad r_t^{\mathrm W}=\sqrt2\,(5-0.992^t),\qquad r_t^{\mathrm H}=4\sqrt2.
\end{equation}

\subsection{Proof of Theorem~\ref{thm:gs}}\label{app:sec34gsgrowth}
\begin{proof}
The projection gives $\langle W_t,\widetilde O_t\rangle_F=0$ and $\|\widetilde O_t\|_F^2=\|O_t\|_F^2-\langle W_t,O_t\rangle_F^2/r_t^2\le p$. Hence
\[
r_{t+1}^2=\|W_t-\eta_t\widetilde O_t\|_F^2=r_t^2+\eta_t^2\|\widetilde O_t\|_F^2,
\]
so the iterates stay nonzero, and every subsequent projection is defined. Summing over $t$ gives
\[
r_T^2
=r_0^2+\sum_{t=0}^{T-1}\eta_t^2\norm{\widetilde O_t}^{\,2}
\le r_0^2+p\sum_{t=0}^{T-1}\eta_t^2.
\]
This proves the radius bound \eqref{eq:gs_radius}.

For $\eta_t=\eta$, \eqref{eq:gs_radius} gives $r_t\le\sqrt{r_0^2+p\eta^2t}$, which increases with $t$; comparing the sum with an integral,
\[
C_T\ge\sum_{t=0}^{T-1}\frac{\eta}{\sqrt{r_0^2+p\eta^2t}}\ge\int_0^T\frac{\eta\,dx}{\sqrt{r_0^2+p\eta^2x}}=\frac{2}{p\eta}\Bigl(\sqrt{r_0^2+p\eta^2T}-r_0\Bigr).
\]
For fixed $r_0$, $p$ and $\eta$, the right-hand side equals $2\sqrt{T/p}+O(1)$, proving \eqref{eq:gsacc}.

For $T\le t<2T$, the radius bound gives $r_t\le\sqrt{r_0^2+2p\eta^2T}$. Hence
\[
C_{2T}-C_T
=\sum_{t=T}^{2T-1}\frac{\eta}{r_t}
\ge\frac{\eta T}{\sqrt{r_0^2+2p\eta^2T}},
\]
and the right-hand side equals $\sqrt{T/(2p)}+O(1)$, which proves \eqref{eq:gswindow}.
\end{proof}

\subsection{Proof of Proposition~\ref{prop:stationarity}}
\begin{proof}
Write $\bar W_t=W_t/r_t$ and $G_t=\operatorname{grad} f(\bar W_t)$. The chain rule gives $\nabla F(W_t)=G_t/r_t$, so $O_t=\mathrm{Orthogonalize}(G_t)$. Since $G_t$ is tangent, $\langle G_t,\bar W_t\rangle_F=0$, and
\[
\langle G_t,\widetilde O_t\rangle_F=\langle G_t,O_t\rangle_F=\|G_t\|_*.
\]
Here $\|\cdot\|_*$ denotes the nuclear norm. When $G_t\ne0$, Cauchy--Schwarz gives
\[
 1\le\frac{\|G_t\|_*}{\|G_t\|_F}
 \le\|\widetilde O_t\|_F\le\|O_t\|_F\le\sqrt p.
\]
If $G_t=0$ for some $t<T$, all later iterates are stationary and~\eqref{eq:stationarity} holds. Otherwise~\eqref{eq:gs_radius} gives
\begin{equation}\label{eq:gs_radius_2}
    r_0^2+\eta^2t\le r_t^2\le r_0^2+p\eta^2t,\qquad0\le t\le T.
\end{equation}

Let $L_1$ and $L_2$ bound the sphere-gradient norm and the operator norm of the Riemannian Hessian of $f$. These bounds are finite because $f$ is $C^2$ on the compact unit sphere. At a nonstationary step, set
\[
 u=\frac{\widetilde O_t}{\|\widetilde O_t\|_F},\qquad
 x=\frac{\eta\|\widetilde O_t\|_F}{r_t},\qquad
 \theta=\arctan x.
\]
Orthogonality gives $\bar W_{t+1}=\bar W_t\cos\theta-u\sin\theta$. Taylor's theorem along this great-circle arc yields
\[
 f(\bar W_{t+1})\le f(\bar W_t)-\theta\langle G_t,u\rangle_F
                        +\frac{L_2}{2}\theta^2.
\]
Since $x-\arctan x=\int_0^x s^2/(1+s^2)\,ds\le x^2/2$,
using $\theta\le x$ and
$0\le\langle G_t,u\rangle_F\le L_1$, we obtain
\begin{equation}\label{eq:one-step}
    f(\bar W_{t+1})\le f(\bar W_t)-x\langle G_t,u\rangle_F+\frac{L_1+L_2}{2}\,x^2
 \le f(\bar W_t)-\frac{\eta}{r_t}\|G_t\|_*
                      +B\frac{\eta^2}{r_t^2}
\end{equation}
where $B=p(L_1+L_2)/2$. If a stationary iterate occurs before $T$, the claimed minimum is zero. Otherwise, the radius bounds in \eqref{eq:gs_radius_2} imply
\[
 \sum_{t=T/2}^{T-1}\frac{\eta^2}{r_t^2}\le1,
 \qquad
 \sum_{t=T/2}^{T-1}\frac{\eta}{r_t}
 \ge\frac{\eta T}{2\sqrt{r_0^2+p\eta^2T}}.
\]
Summing the one-step estimate \eqref{eq:one-step} over interval $T/2\le t < T$ and using
$\|G_t\|_F\le\|G_t\|_*$ gives
\[
 \min_{T/2\le t<T}\|G_t\|_F
 \le\frac{2\sqrt{r_0^2+p\eta^2T}}{\eta T}
       \bigl(\max f-\min f+B\bigr)
 =O(T^{-1/2}),
\]
where the extrema are over the unit sphere and all constants are independent of $T$.
\end{proof}

\section{Dependence on the Nominal Learning Rate}
\label{app:lr}
Section~\ref{sec:control} asks whether a larger nominal learning rate can make up for the small accumulation of vanilla Muon. This section shows that it cannot beyond the first step. In particular, we vary the constant learning rate $\eta$ while keeping the initialization and the norm-control parameters fixed. Note that the directions $O_t$ and the norms $r_t$ may depend on $\eta$. The first term $C_1=\eta/r_0$ is linear in $\eta$ for every method, so we compare the accumulation after the first step,
\[
C_T(\eta)-C_s(\eta)=\sum_{t=s}^{T-1}\frac{\eta}{r_t(\eta)},\qquad1\le s<T.
\]

\paragraph{Vanilla Muon.} For vanilla Muon, a larger learning rate enlarges every update but may also accelerate norm growth. The next proposition shows that the two effects offset each other when the updates keep an outward component and, with a weaker bound, whenever they have no inward component.
 
\begin{proposition}[Learning-rate compensation]\label{prop:lr}
Consider vanilla Muon with a constant nominal learning rate $\eta$, and let $1\le s<T$.
\begin{enumerate}
\item[(i)] If
\begin{equation}\label{eq:outward}
-\frac{\langle W_t,O_t\rangle_F}{r_t}\ge v>0,\qquad0\le t<T,
\end{equation}
then $r_0+\eta vt\le r_t\le r_0+\eta\sqrt p\,t$ for $0\le t\le T$. Moreover,
\begin{equation}\label{eq:lrbounds}
\sum_{t=s}^{T-1}\frac{1}{r_0/\eta+\sqrt p\,t}\le C_T-C_s\le\sum_{t=s}^{T-1}\frac{1}{r_0/\eta+vt}\le\frac1v\sum_{t=s}^{T-1}\frac1t.
\end{equation}
\item[(ii)] If $\langle W_t,O_t\rangle_F\le0$ and $\|O_t\|_F\ge \mu>0$ for $0\le t<T$, then $r_t^2\ge r_0^2+\mu^2\eta^2t$ for $0\le t\le T$. Moreover,
\begin{equation}\label{eq:lr-tangent}
C_T-C_s\le\sum_{t=s}^{T-1}\frac{\eta}{\sqrt{r_0^2+\mu^2\eta^2t}}\le\frac{1}{\mu}\sum_{t=s}^{T-1}\frac{1}{\sqrt t}.
\end{equation}
\end{enumerate}
\end{proposition}
 
\begin{proof}
(i) Consider the direction along $W_t/r_t$. The norm $r_{t+1}$ is at least the component of $W_{t+1}=W_t-\eta O_t$ along this direction, which gives
\[
r_{t+1}\ge \left\langle \frac{W_t}{r_t}, W_t-\eta O_t\right\rangle_F =r_t-\eta\langle W_t,O_t\rangle_F/r_t\ge r_t+\eta v.
\]
Meanwhile, the triangle inequality gives $r_{t+1}\le r_t+\eta\sqrt p$. Summing these increments gives \eqref{eq:outward}. Dividing $\eta$ by these bounds and summing over $s\le t<T$ gives~\eqref{eq:lrbounds}.
 
(ii) Since $\langle W_t,O_t\rangle_F\le0$,
\[
r_{t+1}^2=r_t^2-2\eta\langle W_t,O_t\rangle_F+\eta^2\|O_t\|_F^2\ge r_t^2+\mu^2\eta^2.
\]
Summing gives the norm bound, and dividing $\eta$ by it gives~\eqref{eq:lr-tangent} in the same way.
\end{proof}
 
Proposition~\ref{prop:lr} shows why a larger learning rate cannot raise the accumulation of vanilla Muon after the first step: the norm grows in proportion to $\eta$. In (i), every update has an outward component, and $r_t\ge\eta vt$; in (ii), tangent updates enlarge the norm through the second-order term $\eta^2\|O_t\|_F^2$, and $r_t\ge \mu\eta\sqrt t$. A larger $\eta$ thus enlarges the steps and the norm by the same factor, so the effective learning rates after the first step are bounded independently of $\eta$:
\[
\frac{\eta}{r_t}\le\frac{1}{vt}\quad\text{in (i)},\qquad\frac{\eta}{r_t}\le\frac{1}{\mu\sqrt t}\quad\text{in (ii)},\qquad t\ge1.
\]
Summing them gives the last bounds in~\eqref{eq:lrbounds} and~\eqref{eq:lr-tangent}: whatever the value of $\eta$, $C_T-C_s$ stays below $v^{-1}\sum_{t=s}^{T-1}1/t$ in (i) and below $\mu^{-1}\sum_{t=s}^{T-1}t^{-1/2}$ in (ii). The middle sums in~\eqref{eq:lrbounds} and~\eqref{eq:lr-tangent} do increase with $\eta$, but they converge to these limits, so a larger $\eta$ can bring $C_T-C_s$ at most up to a level that does not depend on $\eta$. Only the first step, taken before the norm has grown, gains from a larger $\eta$, as $C_1=\eta/r_0$ grows linearly. 
 
The same holds when $\eta$ scales a fixed schedule $\eta_t=\eta g_t$ with $g_0>0$ as in practical training. In particular, with $S_t=\sum_{j<t}g_j$ and $Q_t=\sum_{j<t}g_j^2$, the effective learning rates satisfy $\eta_t/r_t\le g_t/(vS_t)$ in (i) and $\eta_t/r_t\le g_t/(\mu\sqrt{Q_t})$ in (ii) for $t\ge1$, again independent of $\eta$.
 
In contrast, norm control approaches bound the norm $r_t$ by a quantity that does not depend on $\eta$, so the accumulation grows linearly with $\eta$, exactly for a fixed norm and as a guarantee for MuonW. Specifically, at a fixed radius $R$, $C_T-C_s=\eta(T-s)/R$; and for MuonW with fixed $\lambda>0$ and $0<\eta\lambda<1$, the recursion $r_{t+1}\le(1-\eta\lambda)r_t+\eta\sqrt p$ gives $r_t\le R_\lambda=\max\{r_0,\sqrt p/\lambda\}$, and hence $C_T-C_s\ge\eta(T-s)/R_\lambda$. With a schedule, $\eta(T-s)$ becomes $\eta\sum_{t=s}^{T-1}g_t$ in both, for MuonW when $\eta\lambda g_t\le1$.
 
\section{Dependence on Scale and the Matrix Size}
\label{app:scale}

In the previous sections, the value of $C_T$ depends on the initialization scale, the matrix size and on how Muon's updates are scaled, while the orders of growth that our bounds guarantee do not. Values can therefore be compared between runs that share the initialization and the update scale, as in Figures~\ref{fig:CT_horizon} and~\ref{fig:CT_lr}, whose runs differ only in the norm control, $\eta$ and $\lambda$. Comparisons across initializations or widths need the normalization below. Below we assume $\eta_t=\eta>0$.

\paragraph{Scale and growth.}
Since $C_1=\eta/r_0$,
\begin{equation}
    \frac{C_T}{C_1}
    =\sum_{t=0}^{T-1}\frac{r_0}{r_t}.
    \label{eq:normalized-accumulation}
\end{equation}
Thus the factor $C_1$ carries the learning rate, the initial norm, and any constant that an implementation maintains throughout training. The ratio $C_T/C_1$ measures accumulation in units of the first step and depends only on the normalized norms $r_t/r_0$. Because $C_1$ does not depend on $T$, $C_T$ and $C_T/C_1$ grow at the same rate in $T$.

\paragraph{Normalized accumulation bounds.}
Let
\[
    \alpha=\frac{\eta\sqrt p}{r_0}.
\]
This quantity bounds the initial update size relative to the weight norm. By Proposition~\ref{prop:controls} and Theorem~\ref{thm:gs}, we have
\begin{align}
    \text{vanilla Muon:}\qquad
    \frac{C_T}{C_1}
    &\ge
    \frac{\log(1+\alpha T)}{\alpha},
    \label{eq:normalized-vanilla}\\
    \text{MuonGS:}\qquad
    \frac{C_T}{C_1}
    &\ge
    \frac{2}{\alpha^2}
    \left(\sqrt{1+\alpha^2T}-1\right),
    \label{eq:normalized-gs}\\
    \text{MuonW:}\qquad
    \frac{C_T}{C_1}
    &\ge
    \frac{T}{\max\{1,\alpha/(\eta\lambda)\}},
    \label{eq:normalized-wd}\\
    \text{fixed norm:}\qquad
    \frac{C_T}{C_1}
    &=T.
    \label{eq:normalized-fixed}
\end{align}
The first two follow by comparing the sum of $r_0/r_t$ with an integral, and the third from $R_\lambda/r_0=\max\{1,\alpha/(\eta\lambda)\}$. For every $\alpha>0$, the right-hand sides grow like $\log T$, $\sqrt T$, $T$ and $T$. The value of $\alpha$ changes their constants and the time at which they stop growing linearly, about $\alpha T\approx1$ for vanilla Muon and $\alpha^2T\approx1$ for MuonGS, but not their orders in $T$.

These expressions show that matching $\alpha$ gives the same normalized bounds for vanilla Muon and MuonGS across matrix sizes. MuonW may additionally require matching $\eta\lambda$.

\paragraph{Independence of matrix size for Muon.}
With the usual initialization, a matrix of shape $m_d\times n_d$ with independent $N(0,\sigma^2/n_d)$ entries has $r_{0,d}^2\approx m_d\sigma^2$, so $r_{0,d}\approx\sigma\sqrt{m_d}$, while the Muon update size is $\sqrt{p_d}=\sqrt{\min\{m_d,n_d\}}$. Therefore, 
\[
\alpha_d=\eta\sqrt{p_d}/r_{0,d}\approx \frac{\eta}{\sigma}\sqrt{\frac{p_d}{m_d}}
\]
is asymptotically independent of $d$. Consequently, at a fixed ratio of $p/m$, the normalized Muon -type bounds \eqref{eq:normalized-vanilla}-\eqref{eq:normalized-wd} are asymptotically independent of absolute matrix size at a fixed horizon.

\end{document}